\documentclass[twoside]{article}

\usepackage{arxiv}
\usepackage{alias}
\usepackage[round]{natbib}
\usepackage[utf8]{inputenc}
\usepackage{amsmath}
\usepackage{amsfonts}
\usepackage{amssymb}
\usepackage{amsthm}
\usepackage{xcolor}
\usepackage[framemethod=TikZ]{mdframed}
\usepackage{thm-restate}
\usepackage{tabularx}
\usepackage{booktabs}
\usepackage{mathrsfs}
\usepackage{graphicx}
\usepackage{graphics}
\usepackage{hyperref}
\usepackage{enumitem}
\usepackage{pgfplots}

\usepackage[nodisplayskipstretch]{setspace}

\title{Implicit bias of any algorithm: bounding bias via margin}
\author{Elvis Dohmatob \email{e.dohmatob@criteo.com}\\\addr{Criteo AI Lab}
 }

\begin{document}






\maketitle

\begin{abstract}
Consider $n$ points $\x_1$,\ldots,$\x_n$ in finite-dimensional euclidean space, each having one of two colors. Suppose there exists a separating hyperplane (identified with its unit normal vector $\w)$ for the points, i.e a hyperplane such that points of same color lie on the same side of the hyperplane. 
We measure the quality of such a hyperplane by its margin $\gamma(\w)$, defined as minimum distance between any of the points $\x_i$ and the hyperplane. In this paper, we prove that the margin function $\gamma$ satisfies a nonsmooth Kurdyka-\L{}ojasiewicz inequality with exponent $1/2$. This result has far-reaching consequences. For example, let $\gamma^{opt}$ be the maximum possible margin for the problem and let $\w^{opt}$ be the parameter for the hyperplane which attains this value. Given any other separating hyperplane with parameter $\w$, let $d(\w):=\|\w-\w^{opt}\|$ be the euclidean distance between $\w$ and $\w^{opt}$, also called the bias of $\w$. From the previous KL-inequality, we deduce that $(\gamma^{opt}-\gamma(\w)) / R \le d(\w) \le 2\sqrt{(\gamma^{opt}-\gamma(\w))/\gamma^{opt}}$, where $R:=\max_i \|\x_i\|$ is the maximum distance of the points $\x_i$ from the origin. Consequently, for any optimization algorithm (gradient-descent or not), the bias of the iterates converges at least as fast as the square-root of the rate of their convergence of the margin. Thus, our work provides a generic tool for analyzing the implicit bias of any algorithm in terms of its margin, in situations where a specialized analysis might not be available: it is sufficient to establish a good rate for converge of the margin, a task which is usually much easier.
\end{abstract}

\section{Introduction}

All through this manuscript, $\mathbb R^m$ will be equipped with the euclidean / $\ell_2$-norm, which we will simply write, $\|\cdot\|$ (without the subscript $2$). We consider binary classification problems with data
$(\x_1,y_1),\ldots,(\x_n,y_n)$ drawn from an unknown distribution on $\mathbb R^m \times \{\pm 1\}$. For each $i \in [n]$, $y_i \in \{\pm 1\}$ is the label and $\x_i \in \mathbb R^m$ are the features of the $i$th example. For simplicity, we will assume $\|\x_i\| \le 1$ for all $i \in [i]$. The integer $n\ge 1$ is the sample size, while $m$ is the dimensionality of the problem. Let $\mathbb S_{m-1} := \{\w \in \mathbb R^m \mid \|\w\|_2 = 1\}$ is the $(m-1)$-dimensional unit-sphere.

We are interested
in "large margin" linear classifiers. Any such model is indexed by a unit-vector $\w \in \sphere$. The prediction on an input example $\x \in \mathbb R^m$ is $\mbox{sign}(\x^\top\w) \in \{\pm 1\}$, where $\x^\top\w$ is the inner product of $\x$ and $\w$ which we will also interchangeably denote by $\langle \x,\w\rangle$. 
The margin of any $\w \in \sphere$, denoted
$\gamma(\w)$, defined by
\begin{eqnarray}
\gamma(\w) := \min_{i \in [n]} y_i\x_i^\top\w.
\label{eq:gamma}
\end{eqnarray}
This measures the minimum (signed) distance of the samples $\x_1,\ldots,\x_n$ to the induced hyperplane $\w^\perp := \{\x \in \mathbb R^m \mid \x^\top\w = 0\}$. Consider the optimal / maximum margin $\gammaopt \in [0, 1]$ for the problem, defined by
\begin{eqnarray}
\gammaopt := \max_{\w \in \mathbb S_{m-1}} \gamma(\w) = \max_{\w \in \mathbb S_{m-1}}\min_{i
  \in [n]} y_i\x_i^\top\w.
  \label{eq:gammaopt}
\end{eqnarray}
This is the maximum possible margin attainable by a linear classifier on the problem.
We will assume that the problem is (linearly) separable, meaning that $\gammaopt > 0$.
Finally, let $\wopt := \arg\max_{\w \in \mathbb S_{m-1}} \gamma(\w) := \{\w \in \mathbb S_{m-1} \mid \gamma(\w) = \gammaopt\}$ be the max-margin model (unique).



\subsection{Summary of main contributions}
Our main contributions can be summarized as follows
\begin{itemize}
\item With a certain nonsmooth replacement of the notion of gradient norm (namely the so-called \emph{strong slope} \cite{degiorgi80}), we prove in Theorem \ref{thm:gold} that the function $f:\mathbb R^m \to \mathbb R\cup \{+\infty\}$ defined by
$$
f(\w) := \begin{cases}-\gamma(\w),&\mbox{ if }\w \in \sphere,\\+\infty,&\mbox{ else.}
\end{cases}
$$
satisfies a \emph{Kurdyka-\L{}ojasiewicz inequality} with exponent 1/2 around the max-margin model $\wopt$, on the unit-sphere $\sphere$. A highlight of this result is that it hints on the possibility of the existence of very fast (perhaps quasi-linear time) algorithms for finding the max-margin model on separable data. These algorithms need not necessarily be gradient-descent in the usual sense. Indeed "sufficient descent" conditions together with KL-inequalities of exponent 1/2 around critical points, are known to lead to linear-time algorithms (even for nonsmooth objectives) ~\cite{Attouch2009}.
\item For our second contribution, we prove in Theorem \ref{thm:main} that
\begin{eqnarray}
\frac{\gammaopt - \gamma(\w)}{R} \le \|\w-\wopt\| \le 2\sqrt{\frac{\gammaopt - \gamma(\w)}{\gammaopt}}.
\label{eq:cinema}
\end{eqnarray}
where $R  = \max_i \|\x_i\|$. These inequalities are graphically illustrated in Figure \ref{fig:figure}. Of course, the LHS is trivial since gamma is Lipschitz w.r.t $\w$. Consequently, for any optimization algorithm (gradient-descent or not), the \emph{bias} $\|\w(t)-\wopt\|$ of the iterates $\w(t)$ converges at least as fast as the square-root of the rate of their convergence of the margin (deficit of) $\gammaopt - \gamma(\w(t))$. Thus, our work provides a generic tool for analyzing the implicit bias of any algorithm in terms of its margin. This can be especially useful in situations where a specialized analysis might not be available; it is then good enough to establish a rate for convergence of the margin, a task which is usually much easier, and then convert it via \eqref{eq:cinema} to a rate of convergence for the bias.
\end{itemize}
\begin{figure}[!htb]
    \centering
    \includegraphics[width=0.7\linewidth]{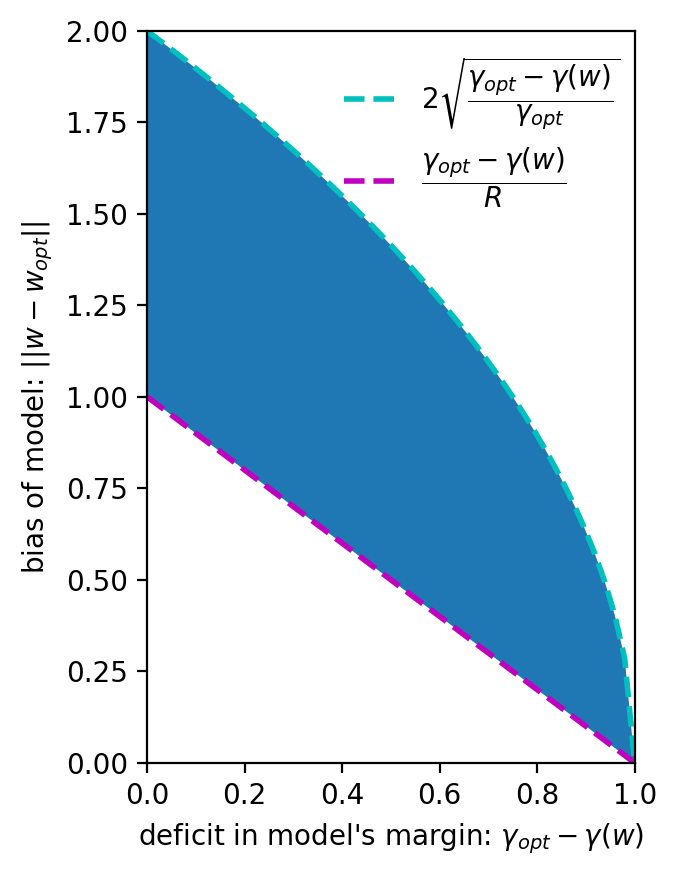}
    \caption{Graphical illustration of \eqref{eq:cinema}. For simplicity, the values in on the figure axes assume $\gammaopt = R = 1$. The result, developed in Theorem \ref{thm:gold}, states that the plots of the margin deficit versus the bias of the iterates generated by any algorithm must lie within the shaded area. This allows one to transfer convergence rates for the margin to of iterates generated by any optimization algorithm, to convergence rates for the bias and vise versa. 
    }
    \label{fig:figure}
\end{figure}
\subsection{Related works}
There is a rich body of research on understand the limiting dynamics of the iterates generated by gradient-descent, i.e the so-called so-called \emph{implicit bias} of the latter.
In the case of linear models with exponential-tailed losses, ~\cite{soudry2017},~\cite{soudry2018}, ~\cite{Gunasekar2018}, ~\cite{telgarsky2019,telgarsky2020} make up the standard literature. These papers all prove that the iterates of gradient-descent on linearly separable binary classification problems converge to the max-margin linear classifier $\wopt$ with margin 
$\gammaopt$. They also contain explicit rates of convergence.
The very recent work ~\cite{telgarsky2020}
establishes a convergence rate of $\mathcal O(1/t)$ for both the margin the bias of gradient-descent. More precisely, the authors show that gradient-descent on exponentially-tailed loss functions and aggressive stepsizes converges with rate $\mathcal O(1/t)$ in both the margin and the bias. These convergence rates are the best known currently in the literature.

Finally, in the case of neural network classifiers, let us mention
~\cite{chizat2020} which analyzes gradient-descent on neural networks with one hidden-layer with logistic loss function, and ~\cite{lyu2020} which studies deep neural networks with positive-homogeneous activation functions (e.g RELU) and exponential-tail loss functions.

\section{Main results}

\label{sec:main}
\subsection{Preliminaries on nonsmooth error bounds}
Central to our paper will be the notions of \emph{strong slope} ~\cite{degiorgi80} and generalized nonsmooth \emph{Kurdyka-\L{}ojasiewicz inequalities} ~\cite{Attouch2009,hoffmanlsc,nonlinearhoffman,bolteOT}. These concepts are now standard in optimization.
\begin{restatable}[Nonsmooth Kurdyka-\L{}ojasiewicz inequalites via strong slope]{df}{}
\label{df:KL}
Let $M=(M,d)$ be a complete metric space. An extended-value function $f:M \to \mathbb R \cup \{\pm \infty\}$ is said to satisfy a generalized Kurdyka-\L{}ojasiewicz inequality with exponent $\theta > 0$ and modulus $\alpha>0$, around the point $\w_0 \in  M$ if there exists $\nu > 0$ and such that
\begin{eqnarray}
  \sslope f(\w) \ge \frac{\alpha}{\theta}(f(\w)-f(\w_0))^{1-\theta}\;\forall \w \in M\text{ with }f(\w_0) < f(\w) < f(\w_0) + \nu.
  \label{eq:KL}
\end{eqnarray}
\end{restatable}
Here, $\sslope f(\w)  \in [0,+\infty]$ is the strong slope \cite{degiorgi80,hoffmanlsc,nonlinearhoffman} of $f$ at the point $\w$, a "synthetic" lower-bound of the the rate of change of $f$ at $\w$ in any direction, defined by
\begin{eqnarray}
\sslope f(\w) := \limsup_{\w' \to \w}\frac{(f(\w)-f(\w'))_+}{d(\w,\w')}.
\end{eqnarray}
Our interest in strong slopes and KL-inequalities is motivated by the following result from ~\cite[Corollary 5.1]{nonlinearhoffman}, which will be the main workhorse for proving our theorems.
\begin{restatable}[Nonlinear error-bound via Kurdyka-\L{}ojasiewicz]{prop}{}
Let $M$ be a complete metric space and $f:M \to \mathbb R\cup\{+\infty\}$ be a proper l.s.c function which satisfies a KL-inequality around $\w_0$ with exponent $\theta > 0$ and other parameters as in \eqref{eq:KL}. Then we have the error bound
\begin{eqnarray}
\dist(\w,\{f \le f(\w_0)\}) \ge \frac{(f(\w)-f(\w_0))^\theta}{\alpha},\;\forall \w \in M \text{ with } f(\w_0) < f(\w) < f(\w_0) + \nu.
\label{eq:nonlinearhoffman}
\end{eqnarray}
\label{prop:nonlinearhoffman}
\end{restatable}

Strong slopes are difficult to compute in general. Fortunately, they can be bounded in terms of more familiar quantities. For example, if $M=(M,\|\cdot\|)$ is a Banach space with topological dual $M^\star=(M^\star,\|\cdot\|_\star)$ and $\partial f(\w) \subseteq M^\star$ is the is the \emph{Fr\'echet subdifferential} of $f$ at $\w$, defined by
\begin{eqnarray}
\partial f(\w) := \left\{\w^\star \in M^\star \;\bigg|\; \liminf_{\w' \to \w}\frac{f(\w')-f(\w) - \langle \w^\star,\w'-\w\rangle}{\|\w'-\w\|} \ge 0\right\},
\label{eq:frechetsubdiff}
\end{eqnarray}
then the following bounds hold
\begin{eqnarray}
\liminf_{(\w',f(\w')) \to (\w,f(\w))}\|\partial f(\w')\| \le \sslope f(\w) \le \|\partial f(\w)\|_\star,
\label{eq:sslopebounds}
\end{eqnarray}
where $\|\partial f(\w)\|_\star$ is the minimum norm of subgradients of $F$ at $\w$, i.e $\|\partial f(\w)\|_\star := \inf \{\|\w^\star\|_\star \mid \w^\star \in \partial f\}$. See \cite{nonlinearhoffman}, for example. In particular,
\begin{itemize}
\item If $f$ is convex, then the second inequality in \eqref{eq:sslopebounds} is an equality. In this case, $\partial f$ is given by the familiar formula
$\partial f(\w) := \{\w^\star \in M^\star \mid f(\w') \ge f(\w) + \langle \w^\star,\w'-\w\rangle\;\forall \w' \in M\}$.
Furthermore, if $\w$ is not not a local minimum point of $f$, then both inequalities in \eqref{eq:sslopebounds} are equalities.
\item If $i_S$ is the indicator function of a nonempty subset $S \subseteq M$, then a simple calculation using the definition \eqref{eq:frechetsubdiff} shows that for every $\w \in S$ we have $\partial i_S(\w) := N_S^{\text{Fr\'echet}}(\w)$, the \emph{Fr\'echet normal cone} of $S$ as $\w$, i.e
\begin{eqnarray}
\label{eq:fncone}
 \partial i_S(\w) = N_S^{\text{Fr\'echet}}(\w) := \left\{\w^\star \in M^\star \;\bigg|\; \limsup_{\w' \overset{S}{\to} \w}\frac{\langle \w^\star,\w'-\w)}{\|\w'-\w\|} \le 0\right\},
\end{eqnarray}
where $\w' \overset{S}{\to} \w$ is means that the limit is taken ax $\w'$ tends to $\w$ while staying within $S$.
\end{itemize}


\subsection{Statement of main results}
The following is the first of our main results. All proofs will be provided in section \ref{sec:proofs}.
\begin{mdframed}[backgroundcolor=cyan!10,rightline=false,leftline=false]
\begin{restatable}[Kurdyka-\L{}ojasiewicz inequality for the margin]{thm}{gold}
The extended-value function $f:\mathbb R^m \to \mathbb R \cup \{+\infty\}$ defined by
$$
f(\w) = \begin{cases}-\gamma(\w),&\mbox{ if }\w \in \mathbb S_{m-1},\\+\infty,&\mbox{ else.}\end{cases}
$$
satisfies a KL-inequality \eqref{eq:KL} around the max-margin model $\wopt$, with exponent $\theta=1/2$.
\label{thm:gold}
\end{restatable}
\end{mdframed}
Note that the (negative) margin function $-\gamma$ which is the subject of the above theorem is neither smooth nor convex.
For our second main contribution, we have the following result.
\begin{mdframed}[backgroundcolor=cyan!10,rightline=false,leftline=false]
\begin{restatable}[Bias bounds from margin bounds]{thm}{gold}
For every unit-vector $\w \in \mathbb S_{m-1}$, we have
\begin{eqnarray}
\label{eq:interlace}
\frac{\gammaopt-\gamma(\w)}{R} \le \|\w-\wopt\| \le 2\sqrt{\frac{\gammaopt- \gamma(\w)}{\gammaopt}},
\end{eqnarray}
\label{thm:main}
where $R := \max_{i \in [n]}\|\x_i\|$.
\end{restatable}
\end{mdframed}
\begin{restatable}{rmk}{}
The above theorem, illustrated in Figure \ref{fig:figure}, can be used to convert rates of convergence of function values $\gamma(\w(t)) \to \gammaopt$
produced by any algorithm (e.g gradient descent), to rates of convergence of iterates, i.e $\|\w(t)-\wopt\| \to 0$.
\end{restatable}

Note that factor $2$ in the RHS of \eqref{eq:interlace} is tight. Indeed consider the classification problem with $n=1$ (just one sample point!), $m=2$, $\x_1=(1,0)$ and $y_1=1$. The margin of any $\w \in \mathbb S_{1}$ is $\gamma(\w) = w_1$ which is maximized when $\w=(1,0)$. Thus, $\gammaopt = 1$ and $\wopt = (1,0)$. On the other hand, taking $\w = (0,1)$, we get $\|\w-\wopt\|=2$ and $1-\gamma(\w)/\gammaopt = 1$ since $\gamma(\w) = 0$.




\section{Proof of main results}
\label{sec:proofs}
In this section, we will prove our main results, namely Theorem \ref{thm:main} and \ref{thm:gold}. Before that, we need some auxiliary results which might be of independent interest themselves.

{\bfseries Notations.} Let $\Delta_{n-1} := \{(q_1,\ldots,q_n) \in \mathbb R^n \mid {\scriptstyle\sum}_{i=1}^n q_i = 1, \min_{i \in [n]} q_i \ge 0\}$ be the unit $(n-1)$-dimensional probability simplex. Given a subset $I \subseteq [n]$ of indices, let $\Delta_{n-1}(I) := \{\q \in \Delta_{n-1} \mid {\scriptstyle\sum}_{i \in I}q_i = 1\}$ be the face of $\Delta_{n-1}$ generated by vertices in $I$. The indicator function $i_A$ of a nonempty subset of $\mathbb R^m$ is the function $i_A:\mathbb R^m \to \mathbb R \cup \{+\infty\}$ defined by $i_A(\w) = 0$ if $\w \in A$, and $i_\sphere = +\infty$ else.

\begin{restatable}[Fr\'echet subdifferential of negative margin function]{thm}{}
The extended-value function $f$ in Theorem \ref{thm:gold} has Fréchet subdifferential which satisfies the following inclusion
$$
\partial f(\w) \subseteq \{b\w-{\scriptstyle\sum}_{i=1}^nq_iy_i\x_i \mid \q=(q_1,\ldots,q_n) \in \Delta_{n-1}(I(\w)),\; b \in \mathbb R\}\;\forall \w \in \sphere,
$$
where $I(\w) := \{i \in [n] \mid y_i\x_i^\top\w = \gamma(\w)\}$ is the set of indices of "support vectors" for $\w$.
\label{thm:frechet}
\end{restatable}

\begin{proof}[Proof of Theorem \ref{thm:frechet}]
Let $\A$ be the $n \times m$ matrix with $i$th row $\ma_i:=-y_i\x_i$, and observe that we can decompose $f=g + i_{\mathbb S_{m-1}}$, where $g:\mathbb R^m \to \mathbb R$ is defined by $g(\w):=\max_{i \in [n]}g_i(\w)$, with $g_i(\w) := \ma_i^\top \w$. By the sum-rule for Fr\'echet subdifferentials, we have $\partial f(\w) \subseteq \partial g(\w) + \partial i_\sphere(\w)$ for all $\w \in \sphere$.
Also, by a well-known result for the subdifferential of the pointwise maximum of convex functions (e.g see ~\cite[Corollary 3.6]{VanNgai2002}), one has
\begin{eqnarray*}
\begin{split}
\partial g(\w) &= \conv\left(\cup_{i \in I(\w)}\partial g_i(\w)\right) = \{{\scriptstyle\sum}_{i=1}^n q_i\w_i^\star \mid \q \in \Delta_{n-1}(I(\w)),\;\w_i^\star \in \partial g_i(\w)\;\forall i \in I(\w)\}\\
&= \{-{\scriptstyle\sum}_{i=1}^nq_iy_i\x_i \mid \q \in \Delta_{n-1}(I(\w))\} = \{\A^\top\q \mid \q \in \Delta_{n-1}(I(\w))\}.
\end{split}
\end{eqnarray*}

On the other hand, by \eqref{eq:fncone} and Example 2.6 of ~\cite{bauchkefrechetcone}, it holds $\w \in \sphere$ for all $\w \in \sphere$ that $\partial i_\sphere(\w) = \mathbb R\w := \{b\w \mid b \in \mathbb R\}$, the $1$-dimensional subspace of $\mathbb R^m$ spanned by $\w$. Putting things together then gives the result.
\end{proof}

For the proof of Theorem \ref{thm:main}, we will also need the following elementary result (also see ~\cite{telgarsky2019})
\begin{restatable}{lm}{smalllemma}
For every $\q=(q_1,\ldots,q_n) \in \Delta_{n-1}$, it holds that
$\gammaopt \le \|\sum_{i=1}^nq_i y_i\x_i\| \le 1$.
\label{lm:smalllemma}
\end{restatable}
We are now ready to proof Theorem \ref{thm:main}.
\begin{proof}[Proof of Theorem \ref{thm:main}]
Let $f:\mathbb R^m \to \mathbb R\cup\{+\infty\}$ be the negative margin function appearing in the theorem. Thanks to Theorem \ref{thm:frechet}, we know that $\partial f(\w) \subseteq \{\A^\top\q + b\w \mid \q \in \Delta_{n-1}(I(\w)),\;b\in\mathbb R\}$ for all $\w \in \sphere$. Thus, we may lower-bound the minimum norm of Fr\'echet subgradients of $g$ at any point $\w\in\sphere$ as follows
\begin{eqnarray*}
\begin{split}
\|\partial f(\w)\|^2 &= \inf_{\w^\star \in \partial f(\w)}\|\w^\star\|^2 \ge \inf_{\q \in \Delta_{n-1}(I(\w)),\;b \in \mathbb R}\|\A^\top\q + b\w\|^2,\text{ by Theorem \ref{thm:frechet}}\\
&= \inf_{\q \in \Delta_{n-1}(I(\w))}\|\proj_{\w^\perp}(\A^\top \q)\|^2,\text{ distance between line and origin}\\
&=\inf_{\q \in \Delta_{n-1}(I(\w))}\|\A^\top\q-(\w^\top\A^\top\q)\w\|^2,\text{ orthogonal projection formula}\\
&=\inf_{\q \in \Delta_{n-1}(I(\w))}\|\A^\top\q\|^2 - (\w^\top\A^\top\q)^2,\text{ basic linear algebra}\\
&=\inf_{\q \in \Delta_{n-1}(I(\w))}\|\A^\top\q\|^2 - \gamma(\w)^2,\text{ because }\ma_i^\top\w=-y_i\x_i^\top\w = -\gamma(\w)\;\forall i \in I(\w)\\
&\ge \inf_{\q \in \Delta_{n-1}}\|\A^\top\q\|^2-\gamma(\w)^2,\text{ since }\Delta_{n-1}(I(\w)) \subseteq \Delta_{n-1}\\
&= \gammaopt^2-\gamma(\w)^2,\text{ by Lemma \ref{lm:smalllemma}}\\
&= (\gamma(\w) + \gammaopt)\cdot(-\gamma(\w)+\gammaopt)\\
&=(\gamma(\w) + \gammaopt)\cdot(f(\w)-\min f),\text{ by definition of }f\\
&\ge \gammaopt\cdot(f(\w)-\min f),\text{ since }\gamma(\w) \ge 0\text{ by assumption}.
\end{split}
\end{eqnarray*}
Combining the above inequality with the LHS of \eqref{eq:sslopebounds} then  gives
$$
\sslope f(\w) \ge \liminf_{(\w',f(\w')) \to (\w,f(\w))}\|\partial f(\w')\| \ge \gammaopt^{1/2}\cdot (f(\w)-\min f)^{1/2}.
$$
Thus, the negative margin function $g$ satisfies a KL-inequality around the max-margin model $\wopt$, with exponent $\theta=1/2$ and modulus $\alpha=2/\sqrt{\gammaopt}$ as claimed.
\end{proof}

\begin{proof}[Proof of Theorem \ref{thm:main}]
The LHS of the inequality is trivial since the margin function $\gamma$ is $R$-Lipschitz on $\mathbb S_{m-1}$.
For the RHS, note from Theorem \ref{thm:gold} that the negative margin function $g$ (defined in Theorem \ref{thm:main}) satisfies a KL-inequality around the point $\w_0 = \wopt$, with exponent $\theta=1/2$ and modulus $\alpha=2/\sqrt{\gammaopt}$. The result then follows as upon invoking Proposition \ref{prop:nonlinearhoffman}.
\end{proof}

\section{Concluding remarks}
We have established a Kurdyka-\L{}ojasiewicz inequality with exponent $1/2$ for the margin function in linearly separable classification problems. This result gives hopes for the existence of fast (perhaps quasi-linear) optimization schemes for such problems, a quest which will be pursued in future work. Also, we have employed our result to establish a generic inequality linking the convergence rates of the bias and of margin. This immediately allows for the transfer of convergence rates for the margin, to convergence rates for the bias, irrespective of the algorithms / constructs (gradient-flow, gradient-descent, what stepsize, etc.) used to establish the former.

\paragraph{Acknowledgement.} The author is thankful to Ziwei Ji and Matus Telgarsky for reporting an error in one of the theorems in an earlier version of this preprint. Also, thanks to Eugene Ndiaye for proof-reading the manuscript.

\bibliographystyle{apalike}
\bibliography{literature}

\begin{thebibliography}{}

\bibitem[Attouch and Bolte, 2009]{Attouch2009}
Attouch, H. and Bolte, J. (2009).
\newblock On the convergence of the proximal algorithm for nonsmooth functions
  involving analytic features.
\newblock {\em Mathematical Programming}, 116(1).

\bibitem[Azé and Corvellec, 2017]{nonlinearhoffman}
Azé, D. and Corvellec, J.-N. (2017).
\newblock Nonlinear error bounds via a change of function.
\newblock {\em Journal of Optimization Theory and Applications}, 172.

\bibitem[Bauschke et~al., 2013]{bauchkefrechetcone}
Bauschke, H.~H., Luke, D.~R., Phan, H.~M., and Wang, X. (2013).
\newblock Restricted normal cones and the method of alternating projections:
  Applications.
\newblock {\em Set-Valued and Variational Analysis}, 21(3).

\bibitem[{Bolte} and {Blanchet}, 2016]{bolteOT}
{Bolte}, J. and {Blanchet}, A. (2016).
\newblock {A family of functional inequalities: Lojasiewicz inequalities and
  displacement convex functions}.
\newblock {\em arXiv e-prints}.

\bibitem[Chizat and Bach, 2020]{chizat2020}
Chizat, L. and Bach, F. (2020).
\newblock Implicit bias of gradient descent for wide two-layer neural networks
  trained with the logistic loss.
\newblock In {\em Conference on Learning Theory, {COLT} 2020, 9-12 July 2020,
  Virtual Event [Graz, Austria]}, volume 125 of {\em Proceedings of Machine
  Learning Research}. {PMLR}.

\bibitem[Corvellec and Motreanu, 2007]{hoffmanlsc}
Corvellec, J.-N. and Motreanu, V.~V. (2007).
\newblock Nonlinear error bounds for lower semicontinuous functions on metric
  spaces.
\newblock {\em Mathematical Programming}, 114(2):291.

\bibitem[{De Giorgi} et~al., 1980]{degiorgi80}
{De Giorgi}, E., Marino, A., and Tosques, M. (1980).
\newblock Problemi di evoluzione in spazi metrici e curve di massima pendenza.
\newblock {\em Atti della Accademia Nazionale dei Lincei. Classe di Scienze
  Fisiche, Matematiche e Naturali. Rendiconti}, 68(3):180--187.

\bibitem[{Gunasekar} et~al., 2018]{Gunasekar2018}
{Gunasekar}, S., {Lee}, J., {Soudry}, D., and {Srebro}, N. (2018).
\newblock {Characterizing Implicit Bias in Terms of Optimization Geometry}.
\newblock {\em arXiv e-prints}, page arXiv:1802.08246.

\bibitem[{Ji} and {Telgarsky}, 2019]{telgarsky2019}
{Ji}, Z. and {Telgarsky}, M. (2019).
\newblock {A refined primal-dual analysis of the implicit bias}.
\newblock {\em arXiv:1906.04540 (version v1 of manuscript)}.

\bibitem[{Ji} and {Telgarsky}, 2020]{telgarsky2020}
{Ji}, Z. and {Telgarsky}, M. (2020).
\newblock {Characterizing the implicit bias via a primal-dual analysis}.
\newblock {\em arXiv:1906.04540 (version v2 of manuscript)}.

\bibitem[Lyu and Li, 2020]{lyu2020}
Lyu, K. and Li, J. (2020).
\newblock Gradient descent maximizes the margin of homogeneous neural networks.
\newblock In {\em 8th International Conference on Learning Representations,
  {ICLR} 2020}. OpenReview.net.

\bibitem[{Nacson} et~al., 2018]{soudry2018}
{Nacson}, M., {Lee}, J.~D., {Gunasekar}, S., {Savarese}, P. H.~P., {Srebro},
  N., and {Soudry}, D. (2018).
\newblock {Convergence of Gradient Descent on Separable Data}.
\newblock {\em arXiv e-prints}, page arXiv:1803.01905.

\bibitem[{Soudry} et~al., 2017]{soudry2017}
{Soudry}, D., {Hoffer}, E., {Shpigel Nacson}, M., {Gunasekar}, S., and
  {Srebro}, N. (2017).
\newblock {The Implicit Bias of Gradient Descent on Separable Data}.
\newblock {\em arXiv e-prints}, page arXiv:1710.10345.

\bibitem[{Van Ngai} et~al., 2002]{VanNgai2002}
{Van Ngai}, H., {The Luc}, D., and Th\'era, M. (2002).
\newblock Extensions of fr\'echet $\epsilon$-subdifferential calculus and
  applications.
\newblock {\em Journal of Mathematical Analysis and Applications}, 268(1).

\end{thebibliography}

\appendix
\section{Omitted technical proofs}


\smalllemma*
\begin{proof}
For any $\q=(q_1,\ldots,q_n) \in \Delta_{n-1}$, one computes
\begin{eqnarray*}
\begin{split}
\|\sum_{=1}^, q_iy_i\x_i\| &= \sup_{\bu \in \mathbb B_M} \langle \sum_{i=1}^nq_iy_i\x_i,\bu\rangle = \sup_{\bu \in \mathbb B_M} \mathbb E_{i \sim \q}[y_i\langle \x_i,\bu\rangle]\\
&\ge \mathbb E_{i \sim \q}[y_i\langle \x_i,\w^{\text{opt}}_n\rangle],\text{ by taking any }\bu = \w_n^{\text{opt}} \in \opt_n\\
& \ge \gammaopt,\text{ by definition of }\gammaopt.
\end{split}
\end{eqnarray*}
This proves the lower-bound.
On the other hand, using the fact that $$
\sup_{\bu \in \mathbb B_M} \mathbb E_{i \sim \q}[y_i\langle \x_i, \bu\rangle] \le \mathbb E_{i \sim \q} [\sup_{\bu \in \mathbb B_M} y\langle \x_i, \bu\rangle] = \mathbb E_{i \sim \q}[\|\x\|] \le 1,
$$
since $\|\x_i\| \le 1$ for all $i \in [n]$ by hypothesis. This proves the upper-bound.
\end{proof}

\end{document}